\documentclass[5p,times]{elsarticle}
\usepackage{amsmath,amssymb,amsfonts,amsthm}
\usepackage{graphicx}
\usepackage{booktabs}

\usepackage{microtype}
\newtheorem{theorem}{Theorem}

\journal{Neural Networks}

\newcommand{\mask}[1]{\texttt{#1}}

\begin{document}

\begin{frontmatter}

\title{From Latent Space to Training Data: Explainable Specialization in Minimal MLPs\tnoteref{t1}}
\tnotetext[t1]{Code, data, figures, and tables for full reproduction will be released on manuscript publication.}

\author[itis]{Enrique Alba}
\ead{ealbat@uma.es}

\author[itis]{Ezequiel L\'opez-Rubio}
\ead{ezeqlr@lcc.uma.es}

\affiliation[itis]{organization={ITIS Software, University of Malaga}, country={Spain}}

\begin{abstract}
We here study whether training biases can make hidden neurons specialize in minimal one-hidden-layer MLPs, and whether such specialization improves prototype-based reconstruction of the training dataset from the learned weights. We consider Gaussian-activation MLPs of width equal to dataset size and compare three structural losses that respectively encourage coverage of the training samples, separation between neuron-induced prototypes, and low overlap of hidden responses, against the standard fitting baseline. Experiments on uniformly sampled one-dimensional datasets show a stable pattern from $N=3$ to $N=100$ across $480$ controlled runs. Coverage regularization gives the lowest mean reconstruction error at every tested size and raises the prototype-usage specialization ratio relative to the standard baseline, while separation has mixed effects and overlap penalties are systematically harmful. We show that the harm is not an optimization failure: overlap-active approaches fit the data as well as overlap-free ones but route the optimizer to a degenerate equilibrium in which prototype centers are pushed outside the convex hull of the training inputs. Coverage cannot reward this expulsion and acts as an attractor: separation admits it only at large temperature and overlap admits it at the nominal hyperparameter choice. A direct $\tau$-sweep on the separation-only mask and a prototype-position visualization at $N=100$ confirm the mechanism. The findings yield a simple design principle for prototype-recoverability-aware training: every repulsive structural loss must be compensated by a compatible attractor, or it will collapse the latent geometry it was meant to refine.
\end{abstract}


\end{frontmatter}

\section{Introduction}
\label{sec:intro}

Understanding what training information is encoded in neural-network weights remains a basic question in machine learning. This question matters not only for privacy and data leakage, but also for explainability: if training shapes a latent space with interpretable internal structure, then part of the model's behavior may be understood through that structure rather than only through input--output responses. Representation-learning work has long argued that learning quality depends on how explanatory factors are organized in the internal representation~\cite{bengio2013representation}. Probe-based analyses make the same point from a diagnostic perspective, showing that intermediate layers can be inspected to reveal what information the network makes linearly accessible during learning~\cite{alain2017probes}. More broadly, interpretable machine learning has emphasized that understanding internal mechanisms is often more informative than post hoc explanations attached to an already trained black box~\cite{rudin2019stop}. This perspective also matters for certification: properties such as copyright compliance, fairness, or safety are easier to verify when latent structure is designed to be auditable rather than discovered after the fact.

A nearby line of work studies what training data can be recovered after training. Carlini et al.\ show that large language models may emit memorized training sequences under extraction attacks~\cite{carlini2021extracting}. Haim et al.\ show that, under favorable conditions, substantial fractions of training data may be reconstructed directly from trained network parameters~\cite{haim2022reconstructing}. These results establish that learned weights can retain recoverable training information, but they treat reconstruction as a post hoc problem: they do not ask how training itself might be biased toward hidden representations that are more specialized, more structured, easier to reconstruct or easier to audit against design-time requirements such as copyright compliance or safety constraints.

Another nearby line of work introduces regularizers that encourage hidden units to become less redundant or more diverse. Cogswell et al.\ reduce overfitting by decorrelating learned representations~\cite{cogswell2016decov}, and Oostwal et al.\ analyze hidden-unit specialization across activation functions~\cite{oostwal2021specialization}. Such work shows that explicit penalties can reshape the latent space during training. A second line of work studies the geometry that emerges when training balances \emph{attractive} and \emph{repulsive} forces between weights and features. Xie et al.\ show that under imbalanced classification the cross-entropy gradient on each classifier weight decomposes into an attractive component pulling that weight toward intra-class features and repulsive components pushed by inter-class features, and that imbalance between the two causes the so-called minority-collapse pathology~\cite{xie2023neural}. The structural losses studied in this paper are exactly such attractive and repulsive forces acting on the prototype geometry of a Gaussian MLP, and the failure mode we observe in the prototype-regression setting is closely analogous to that picture. Our setting differs in purpose: we study whether simple structural losses can bias hidden neurons toward specialized roles in a way that improves reconstruction of the training dataset from the learned weights.

The architectural family used here also has a long history. One-hidden-layer networks with Gaussian activations correspond functionally to radial basis function (RBF) networks~\cite{broomhead1988multivariable,moody1989fast}, which are universal approximators on compact sets~\cite{park1991universal}. Each Gaussian unit is parameterised by a center $-b_j/w_j$ and a width that is monotone in $|w_j|$, and these centers are exactly the prototype locations exploited below. The contribution of this paper is therefore not the architecture but the use of explicit structural training biases that act on those prototypes, and the controlled measurement of how those biases interact with reconstruction from weights.

We study that question in the smallest setting where it can be made precise: one-hidden-layer Gaussian MLPs trained on tiny scalar datasets. The goal is not to claim scalable exact recovery from weights. Rather, the goal is to isolate how three structural losses affect hidden-neuron specialization and whether the induced latent structure helps reconstruct the training dataset. This controlled regime makes the interaction between specialization and recoverability visible, exposes where training biases help or harm, and provides a canonical testbed for later extensions to richer architectures and larger learning settings.

The paper addresses three questions. \textbf{RQ1:} Which specialization bias, if any, improves prototype-based reconstruction over standard training? \textbf{RQ2:} Which of the three proposed structural losses is beneficial, neutral, or harmful when activated alone or in combination, and \emph{why}? \textbf{RQ3:} Do the same effects persist as the dataset size moves from the smallest non-trivial cases to larger minimal regimes? Section~\ref{sec:setup} defines the MLP setting and the three losses that make these questions precise. Section~\ref{sec:protocol} fixes the experimental protocol. Section~\ref{sec:results} reports and discusses the results, including an explicit mechanism that explains the qualitative ranking observed across sizes. Section~\ref{sec:limitations} states the scope and limitations of the study. Section~\ref{sec:repro} describes the released software, and Section~\ref{sec:conclusions} concludes.

\section{Minimal MLPs and Specialization Losses}
\label{sec:setup}

This section formalises the minimal MLP setting, the fitting loss, and the three structural losses for our experiments. Let
\begin{equation}
\label{eq:dataset}
D=\{(x_i,y_i)\}_{i=1}^{N},\qquad x_i,y_i\in[0,1],
\end{equation}
be a dataset of $N$ scalar pairs with strictly increasing $x_i$. We use a one-hidden-layer Gaussian MLP of width $H=N$,
\begin{equation}
\label{eq:model}
f_\theta(x)=\sum_{j=1}^{N} a_j\,\sigma(w_jx+b_j)+c,
\qquad \sigma(z)=\exp(-z^2),
\end{equation}
with parameters
\begin{equation}
\theta=\{(a_j,w_j,b_j)\}_{j=1}^{N}\cup\{c\}.
\end{equation}
The Gaussian activation localizes each hidden unit around the one-dimensional prototype
\begin{equation}
\label{eq:prototype}
\hat{x}_j=-\frac{b_j}{w_j},
\end{equation}
which induces a reconstructed output
\begin{equation}
\label{eq:recout}
\hat{y}_j=f_\theta(\hat{x}_j).
\end{equation}
Sorting the reconstructed pairs by increasing $\hat{x}_j$ yields the generated dataset
\begin{equation}
\label{eq:gend}
\hat{D}=\{(\hat{x}_j,\hat{y}_j)\}_{j=1}^{N}.
\end{equation}
We evaluate reconstruction by the permutation-invariant error
\begin{equation}
\label{eq:reconstruction_error}
E(D,\hat{D})=\min_{\pi\in S_N}\frac{1}{N}\sum_{i=1}^{N}\left(|x_i-\hat{x}_{\pi(i)}|+|y_i-\hat{y}_{\pi(i)}|\right),
\end{equation}
where $S_N$ is the symmetric group on $N$ elements. Equation~\eqref{eq:reconstruction_error} is computed exactly by linear-sum (Hungarian) assignment in the implementation, which avoids the bias of any fixed ordering.

Standard training minimizes only the fitting loss
\begin{equation}
\label{eq:lfit}
L_{fit}=\frac{1}{N}\sum_{i=1}^{N}\bigl(f_\theta(x_i)-y_i\bigr)^2.
\end{equation}
We augment it with three structural losses. The overlap loss penalizes several hidden units responding strongly to the same input,
\begin{equation}
\label{eq:lov}
L_{overlap}=\sum_{i=1}^{N}\sum_{1\le j<k\le N}\sigma(w_jx_i+b_j)\,\sigma(w_kx_i+b_k).
\end{equation}
The coverage loss encourages every sample to have a nearby prototype,
\begin{equation}
\label{eq:lcov}
L_{coverage}=\sum_{i=1}^{N}\min_{1\le j\le N}\Bigl(x_i+\tfrac{b_j}{w_j}\Bigr)^2.
\end{equation}
The separation loss discourages prototype collapse,
\begin{equation}
\label{eq:lsep}
L_{separation}=\sum_{1\le j<k\le N}\exp\!\left(-\frac{\bigl(\frac{b_j}{w_j}-\frac{b_k}{w_k}\bigr)^2}{\tau}\right).
\end{equation}
The total loss is
\begin{equation}
\label{eq:total_loss}
L=L_{fit}+m_o\lambda_oL_{overlap}+m_c\lambda_cL_{coverage}+m_s\lambda_sL_{separation},
\end{equation}
where the binary triple $(m_o,m_c,m_s)\in\{0,1\}^3$ defines an \emph{ablation mask}, written throughout as a 3-bit string in the order overlap--coverage--separation: $\mask{000}$ is the standard baseline (fit only), $\mask{010}$ activates only coverage, $\mask{111}$ activates all three structural terms, and so on. The latent space is considered more specialized, in the restricted prototype-usage sense used here, when the nearest-prototype assignments of the training points use more distinct neurons. We summarize this by the nearest-prototype specialization ratio
\begin{equation}
\label{eq:spec}
S=\frac{\bigl|\{\arg\min_{j}(x_i-\hat{x}_j)^2: i=1,\dots,N\}\bigr|}{N}\in[0,1].
\end{equation}

These definitions turn the general question of specialization into a concrete optimisation problem with eight masks and two derived metrics, $E$ and $S$. We can give a mechanical interpretation for the three structural losses that will recur throughout the analysis. Coverage~\eqref{eq:lcov} acts as an \emph{attractive} force: each input $x_i$ pulls its nearest prototype toward itself, and the loss decreases as that prototype moves into the input region. Separation~\eqref{eq:lsep} and overlap~\eqref{eq:lov} act as \emph{repulsive} forces between prototypes: separation through their pairwise distance, overlap through their joint activation on the inputs. We will show in Section~\ref{sec:results} that this attraction--repulsion taxonomy predicts which masks help reconstruction and which hurt it, and that overlap and separation differ in their behavior because they admit different routes to satisfying the constraint. Let us now explain our experiments.

\section{Experimental Protocol}
\label{sec:protocol}

This section fixes the synthetic-data setup, training hyperparameters, and statistical procedure used to compare the eight ablation masks across all dataset sizes. We use only synthetic datasets. For each $N\in\{3,5,10,30,50,100\}$, we sample $5$ independent datasets of $N$ pairs from the uniform distribution on $[0,1]^2$, and for each dataset we train $2$ random initializations under every mask. This yields $6\times 8\times 5\times 2=480$ runs in total.

The input coordinates are sorted and constrained to have a minimum separation. For the core regime $N\in\{3,5,10\}$ that separation is $0.05$. For the larger regimes we use the adaptive rule
\begin{equation}
\label{eq:delta}
\delta_N=\min\{0.05,\,0.5/N\},
\end{equation}
which preserves the small-$N$ setting while keeping larger cases feasible. Sampling under~\eqref{eq:delta} is implemented constructively: $N$ uniform draws on $[0,1-(N-1)\delta_N]$ are sorted and then shifted by $j\delta_N$ in the $j$-th order statistic, which yields a uniform-with-min-separation distribution and is reliable at every $N$.

The Gaussian MLP uses width $H=N$. We train with Adam at learning rate $\eta=10^{-3}$ for $200$ epochs with $\tau=10^{-2}$ and shared coefficients $\lambda_o=\lambda_c=\lambda_s=10^{-2}$ whenever the corresponding term is active. Parameters are initialised as
\begin{equation}
\label{eq:init}
a_j,b_j,c\sim\mathcal{U}(-10^{-1},10^{-1}),\quad
w_j\sim\mathcal{U}(-1,-10^{-1})\cup\mathcal{U}(10^{-1},1).
\end{equation}
After each gradient step we project every hidden weight as $|w_j|\ge 10^{-1}$ to keep the prototype map~\eqref{eq:prototype} numerically well defined.

For every run we record the reconstruction error~\eqref{eq:reconstruction_error}, the specialization ratio~\eqref{eq:spec}, the final values of $L_{fit}, L_{overlap}, L_{coverage}, \allowbreak\ L_{separation}$, and runtime statistics. We report means and standard deviations across the $10$ runs of each $(N,\text{mask})$ configuration. For statistical analysis we first test normality with Shapiro--Wilk and then use a Welch $t$-test when both groups are consistent with normality; otherwise we use a two-sided Mann--Whitney test. Pairwise comparisons are Holm-corrected and interpreted at significance level $\alpha=0.05$.

This protocol makes the loss comparison reproducible and controlled. We now turn to what the experiments reveal about specialization and reconstruction.

\section{Results and Discussion}
\label{sec:results}

We begin with reconstruction accuracy across all six dataset sizes. The binary correspondence is $\mask{000}{=}\text{Std}$, $\mask{001}{=}\text{Sep}$, $\mask{010}{=}\text{Cov}$, $\mask{011}{=}\text{Cov+Sep}$, $\mask{100}{=}\text{Ovl}$, $\mask{101}{=}\text{Ovl+Sep}$, $\mask{110}{=}\text{Ovl+Cov}$, $\mask{111}{=}\text{Full}$. Figure~\ref{fig:error_vs_n} shows all eight masks, and Tables~\ref{tab:ablation_core}--\ref{tab:ablation_scaling} report the corresponding means. The pattern is stable: coverage-only training ($\mask{010}$) gives the lowest mean reconstruction error at every size. The standard baseline ($\mask{000}$) and the coverage-active masks form a clear lower band. Masks that activate the overlap penalty ($\mask{100}$, $\mask{101}$, $\mask{110}$, $\mask{111}$) form an upper band that is consistently worse. This answers RQ1 and RQ2: prototype-based reconstruction can be improved in mean error by an explicit specialization bias, but not every such bias is beneficial.

\begin{figure}[!ht]
\centering
\includegraphics[width=1.0\linewidth]{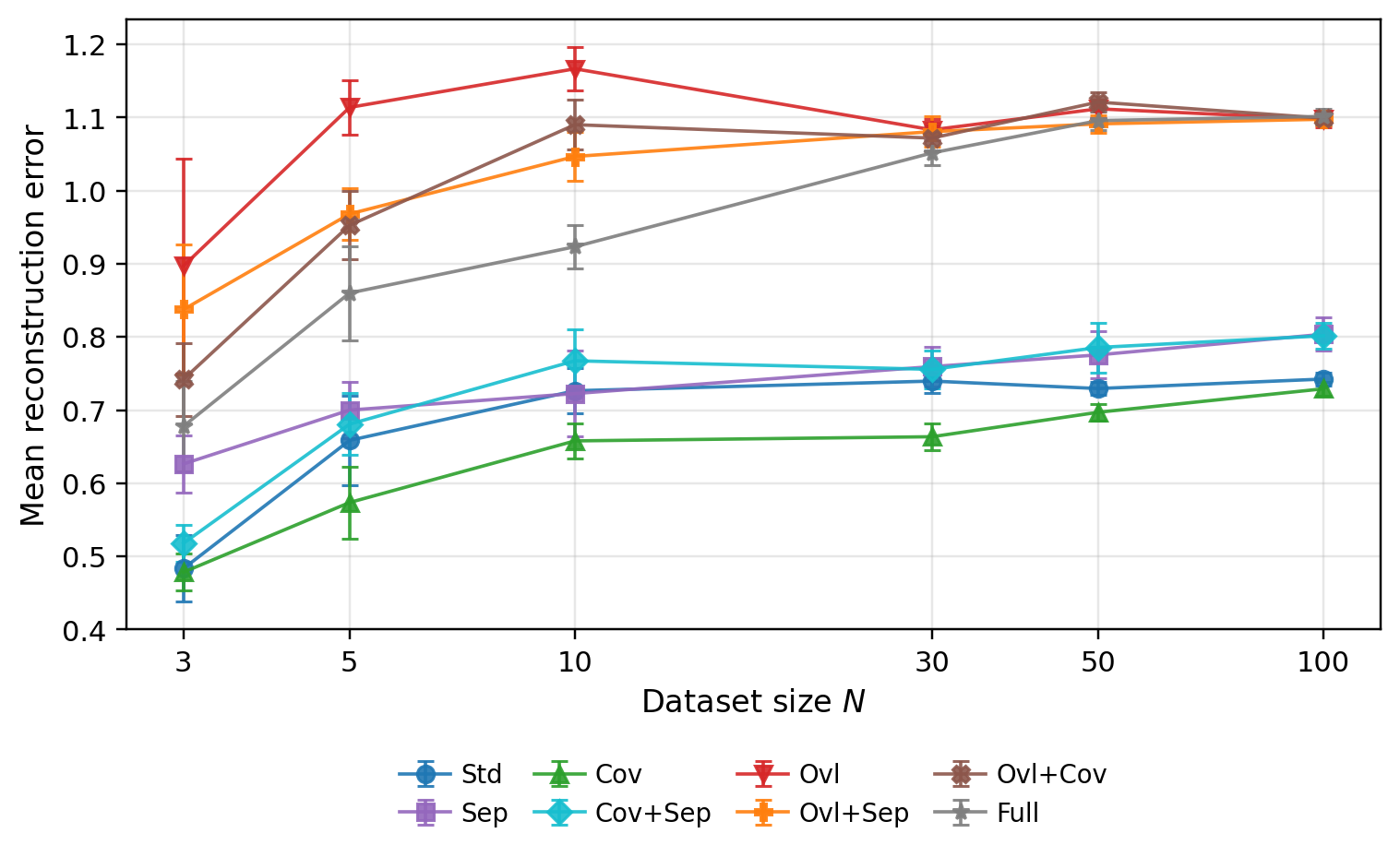}
\caption{Mean reconstruction error versus $N$ for all eight masks; error bars are SEM over $10$ runs per cell. The four coverage-favorable masks (Std, Sep, Cov, Cov+Sep) form a lower band; the four overlap-active masks (Ovl, Ovl+Sep, Ovl+Cov, Full) form a higher band. Coverage-only is the lowest at every $N$.}
\label{fig:error_vs_n}
\end{figure}

\begin{table}[!ht]
\centering
\small
\setlength{\tabcolsep}{4pt}
\caption{Mean reconstruction error for all masks in the core regime. Lower is better. The best value in each row is boldfaced.}
\label{tab:ablation_core}
\begin{tabular}{lrrrrrrrr}
\toprule
$N$ & \texttt{000} & \texttt{001} & \texttt{010} & \texttt{011} & \texttt{100} & \texttt{101} & \texttt{110} & \texttt{111} \\
\midrule
3 & 0.483 & 0.626 & \textbf{0.478} & 0.518 & 0.897 & 0.838 & 0.742 & 0.678 \\
5 & 0.659 & 0.700 & \textbf{0.574} & 0.681 & 1.114 & 0.968 & 0.953 & 0.860 \\
10 & 0.726 & 0.723 & \textbf{0.658} & 0.767 & 1.167 & 1.047 & 1.090 & 0.923 \\
\bottomrule
\end{tabular}

\end{table}

\begin{table}[!ht]
\centering
\small
\caption{Mean reconstruction error for all masks in the scalability regime. Lower is better. The best value in each row is boldfaced.}
\label{tab:ablation_scaling}
{\setlength{\tabcolsep}{4pt}\begin{tabular}{lrrrrrrrr}
\toprule
$N$ & \texttt{000} & \texttt{001} & \texttt{010} & \texttt{011} & \texttt{100} & \texttt{101} & \texttt{110} & \texttt{111} \\
\midrule
30 & 0.740 & 0.760 & \textbf{0.664} & 0.756 & 1.083 & 1.081 & 1.072 & 1.052 \\
50 & 0.729 & 0.775 & \textbf{0.697} & 0.786 & 1.112 & 1.091 & 1.121 & 1.096 \\
100 & 0.742 & 0.804 & \textbf{0.729} & 0.802 & 1.098 & 1.098 & 1.099 & 1.102 \\
\bottomrule
\end{tabular}
}
\end{table}

The heatmap of Figure~\ref{fig:heatmap_error} groups the masks by family rather than by binary index to make the structural split visible at a glance. The figure shows that the ranking is not an artifact of one or two values of $N$: $\mask{010}$ remains the brightest cell across the full range, $\mask{011}$ and $\mask{001}$ are competitive but weaker, and the overlap-based masks are uniformly darker. In particular, the full three-loss objective $\mask{111}$ does not fail because it contains four terms in total; it fails because adding the overlap term systematically worsens reconstruction relative to the corresponding masks without it.

\begin{figure}[!ht]
\centering
\includegraphics[width=1.0\linewidth]{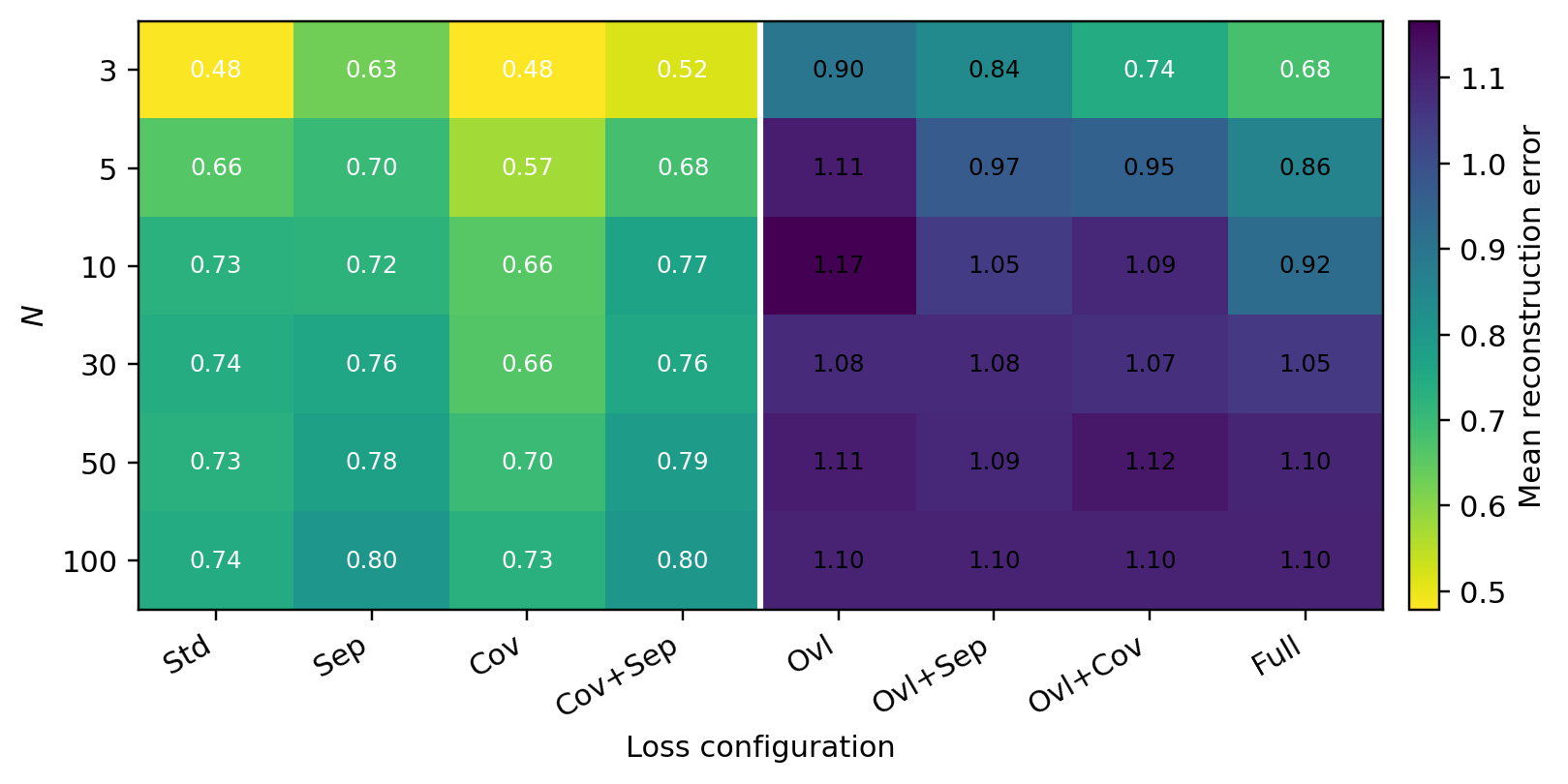}
\caption{Heatmap of mean reconstruction error across masks and dataset sizes (lighter is better). Masks are grouped by family rather than binary index; the white separator marks the boundary between the four overlap-free masks (left) and the four overlap-active masks (right).}
\label{fig:heatmap_error}
\end{figure}

To quantify how large the coverage advantage is, Table~\ref{tab:effect_sizes} reports absolute and relative improvements of $\mask{010}$ over $\mask{000}$ together with Cohen's~$d$ on the per-run reconstruction errors. The effect is non-monotonic in $N$: it is essentially absent at $N=3$ ($d=0.04$, $\sim 1\%$ relative reduction), rises through the mid-range to a peak at $N=30$ ($d=1.40$, $\sim 10\%$), and shrinks again at $N=100$ ($d=0.46$, $\sim 2\%$). The qualitative ordering in Tables~\ref{tab:ablation_core}--\ref{tab:ablation_scaling} therefore tells only part of the story: coverage is reliably best, but its absolute advantage over the baseline is modest at the extremes, and substantial only in a middle regime where the prototype geometry has both enough degrees of freedom to misbehave and enough data points to reward a coverage prior.

\begin{table}[!ht]
\centering
\caption{Effect of coverage-only training ($\mask{010}$) over the standard baseline ($\mask{000}$) on reconstruction error. $\Delta E=E_{\mathtt{000}}-E_{\mathtt{010}}$; relative reduction is $\Delta E/E_{\mathtt{000}}$ in percent; $d$ is Cohen's~$d$ from the per-run distributions.}
\label{tab:effect_sizes}
\begin{tabular}{rrrrrr}
\toprule
$N$ & $E_{\mathtt{000}}$ & $E_{\mathtt{010}}$ & $\Delta E$ & rel.\ red.\ (\%) & Cohen $d$ \\
\midrule
3 & 0.483 & 0.478 & 0.005 & 1.0 & 0.04 \\
5 & 0.659 & 0.574 & 0.085 & 12.9 & 0.49 \\
10 & 0.726 & 0.658 & 0.069 & 9.4 & 0.78 \\
30 & 0.740 & 0.664 & 0.076 & 10.3 & 1.40 \\
50 & 0.729 & 0.697 & 0.032 & 4.4 & 1.01 \\
100 & 0.742 & 0.729 & 0.013 & 1.8 & 0.46 \\
\bottomrule
\end{tabular}

\end{table}

A natural follow-up question is whether overlap masks harm reconstruction by harming optimization itself, or by reshaping the prototype geometry without harming the fit. Table~\ref{tab:fit_loss_family} reports the mean $L_{fit}$ averaged within each family of four masks (overlap-free vs.\ overlap-active). The two families fit the data essentially equally well: the harmful-family fit is $0.98$--$1.09$ times the useful-family fit at every $N$. The reconstruction degradation in Tables~\ref{tab:ablation_core}--\ref{tab:ablation_scaling} therefore comes from prototype placement, not from optimisation failure. This rules out the simplest counter-explanation that overlap penalties merely make training fail.

\begin{table}[!ht]
\centering
\caption{Mean fitting loss $L_{fit}$ averaged within the overlap-free family ($\mask{000}$, $\mask{001}$, $\mask{010}$, $\mask{011}$) and the overlap-active family ($\mask{100}$, $\mask{101}$, $\mask{110}$, $\mask{111}$). The two families fit the data within $\sim 10\%$ of each other; reconstruction differences therefore come from prototype placement, not from optimization failure.}
\label{tab:fit_loss_family}
\begin{tabular}{rrrr}
\toprule
$N$ & $\overline{L_{fit}}$ (no-Ovl) & $\overline{L_{fit}}$ (with-Ovl) & ratio \\
\midrule
3 & 0.0702 & 0.0688 & 0.979 \\
5 & 0.0639 & 0.0687 & 1.076 \\
10 & 0.0849 & 0.0927 & 1.092 \\
30 & 0.0863 & 0.0912 & 1.058 \\
50 & 0.0839 & 0.0886 & 1.056 \\
100 & 0.0873 & 0.0894 & 1.024 \\
\bottomrule
\end{tabular}

\end{table}

\subsection{Why Overlap is Harmful: The Expulsion Mechanism}
\label{sec:mechanism}

Table~\ref{tab:fit_loss_family} shows that overlap-active masks fit the data essentially as well as overlap-free masks but reconstruct it much worse. The harm therefore comes from where the trained model puts its prototypes, not from how well it fits. The shape of the overlap loss explains why. Equation~\eqref{eq:lov} is a sum of pairwise products of activations evaluated at the training inputs. Because each Gaussian activation $\sigma(w_jx+b_j)=\exp(-(w_jx+b_j)^2)$ decays as $|w_jx+b_j|\to\infty$, the optimizer can drive every overlap term containing prototype $\hat{x}_j$ to zero by pushing $\hat{x}_j=-b_j/w_j$ outside the convex hull of the training inputs. An expelled prototype activates negligibly at every training point, so its contribution to every pairwise product vanishes. The loss landscape of \eqref{eq:lov} therefore has a degenerate global minimum at $\sigma(w_jx_i+b_j)\approx 0$ for all $(i,j)$, attainable by simply moving prototypes away from $[0,1]$.

The other two structural terms behave very differently under this lens. Coverage~\eqref{eq:lcov} \emph{cannot} reward expulsion: pushing $\hat{x}_j$ outside $[0,1]$ either makes it the nearest prototype to some input (in which case the squared distance grows and the loss increases) or makes it irrelevant to every input (in which case the loss is unchanged). Coverage is therefore an attractor that holds prototypes inside the input region. Separation~\eqref{eq:lsep} admits the same expulsion route as overlap in principle, but the chosen $\tau=10^{-2}$ controls how strongly. The exponential $\exp\bigl(-(\hat{x}_j-\hat{x}_k)^2/\tau\bigr)$ is essentially zero once $|\hat{x}_j-\hat{x}_k|\gg\sqrt{\tau}\approx 0.1$, so prototypes can satisfy the separation constraint without leaving the input range. We expect, and verify below, that increasing $\tau$ should make separation increasingly expulsive.

Three theorems that characterize the considered structural loss functions are given next, and proved in Appendix A:

\begin{theorem}[Coverage cannot reward expulsion]
\label{thm:coverage_hull_main}
Let $D=\{x_i\}_{i=1}^{N}\subset[0,1]$ be the training inputs and
let $\hat{x}_j=-b_j/w_j$ be the $j$-th prototype.
If $\hat{x}_j\notin[0,1]$ for some $j$, then
moving $\hat{x}_j$ further away from $[0,1]$
does not decrease $L_{\mathrm{coverage}}$.
\end{theorem}

\begin{theorem}[Overlap loss vanishes with separation]
\label{thm:overlap_vanishes_main}
Let $\Delta=\min_{j\neq k}|\hat{x}_j-\hat{x}_k|$ be the minimum
pairwise distance between prototypes, and assume $|w_j|\ge w_{\min}>0$
for all $j$.  Then
\[
  L_{\mathrm{overlap}}\;\le\;
  \binom{N}{2}N\,e^{-w_{\min}^2\Delta^2/2}\;\xrightarrow{\Delta\to\infty}\;0.
\]
\end{theorem}

\begin{theorem}[Separation loss vanishes with separation]
\label{thm:separation_vanishes_main}
Let $\Delta=\min_{j\neq k}|\hat{x}_j-\hat{x}_k|$ and let $\tau>0$.
Then
\[
  L_{\mathrm{separation}}\;\le\;
  \binom{N}{2}\exp\!\Bigl(-\tfrac{\Delta^2}{\tau}\Bigr)
  \;\xrightarrow{\Delta\to\infty}\;0.
\]
\end{theorem}

As seen, Theorem \ref{thm:coverage_hull_main} proves that the coverage structural loss cannot cause the prototypes to be expelled from the convex hull of the training set. In other words, $L_{\mathrm{coverage}}$ cannot lead to unstable network configurations in which the prototypes depart from the training data. On the other hand, Theorems \ref{thm:overlap_vanishes_main} and \ref{thm:separation_vanishes_main} demonstrate that the other two structural losses, $L_{\mathrm{overlap}}$ and $L_{\mathrm{separation}}$, admit expelled configurations as (approximate) global minimizers.

We found direct evidence of the above results in prototype positions at $N=100$. Figure~\ref{fig:expulsion} shows the prototype positions of one median-error run of three masks at $N=100$. Under standard training ($\mask{000}$) and under coverage ($\mask{010}$) about half of the prototypes lie inside the input range $[0,1]$ and the rest cluster just below zero, far enough to be passively inactive but not catastrophically so: the in-hull prototypes carry the reconstruction. Under overlap-only ($\mask{100}$), all $100$ prototypes are pushed outside $[0,1]$, with the entire population sitting between $-1$ and $0$. With every prototype expelled, no in-hull neuron remains to support a faithful reconstruction, and the error rises by $\sim 50\%$ relative to the other two masks.

\begin{figure}[!ht]
\centering
\includegraphics[width=1.0\linewidth]{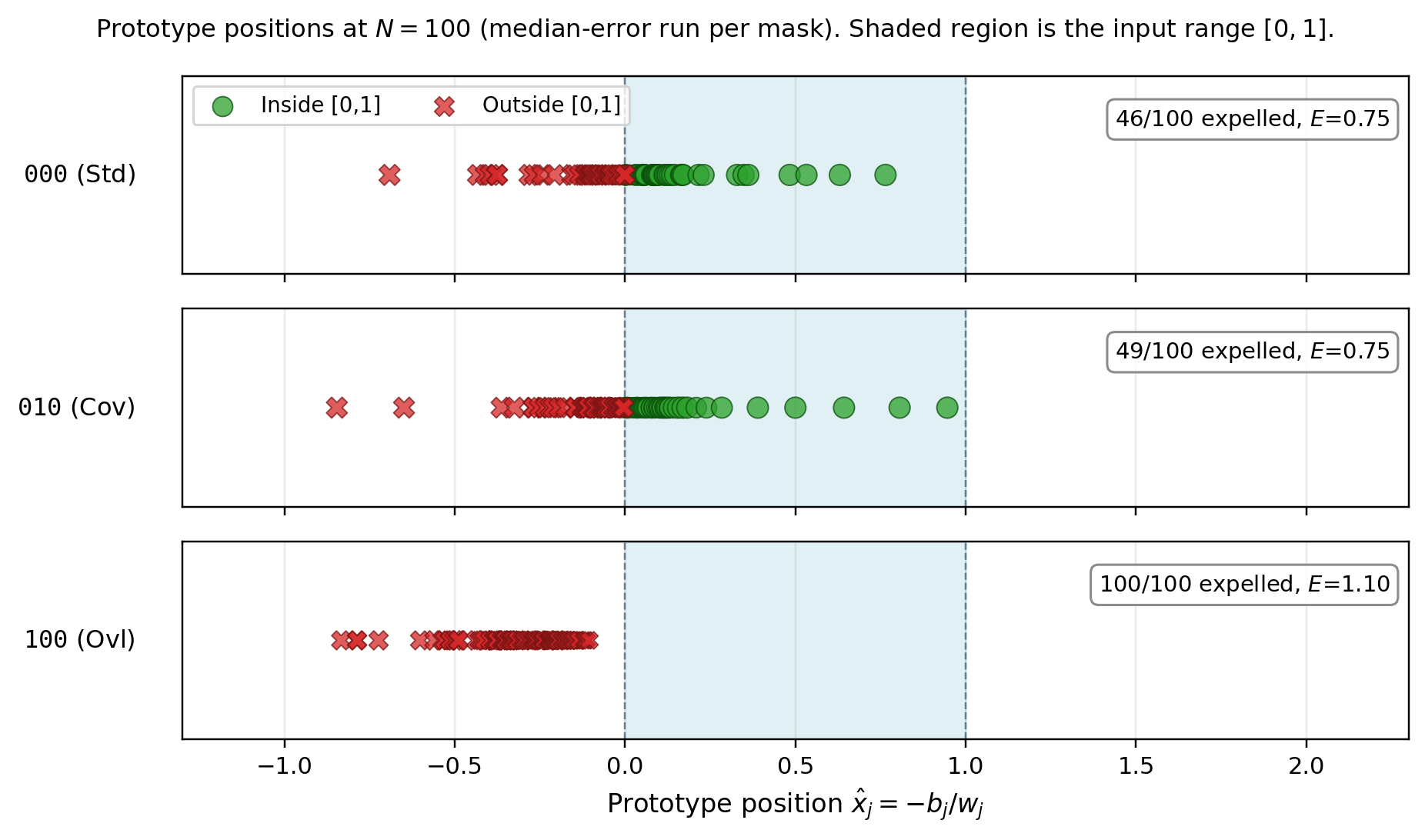}
\caption{Prototype positions $\hat{x}_j=-b_j/w_j$ for one median-error run of each of three masks at $N=100$. The shaded region is the input range $[0,1]$. Marker size is proportional to the prototype's mean activation across the training inputs. Under $\mask{000}$ and $\mask{010}$ a working subset of prototypes remains inside the input range; under $\mask{100}$ all $100$ prototypes are expelled.}
\label{fig:expulsion}
\end{figure}

Table~\ref{tab:expulsion} reports the fraction of prototypes outside $[0,1]$, averaged over the ten runs of each $(N,\text{mask})$ cell. The pattern is sharp: the four overlap-free masks have expulsion fractions clustered between $0.20$ and $0.61$ and bounded away from full expulsion, while the four overlap-active masks reach $1.00$ at every $N\ge 30$ and are above $0.85$ already at $N=10$. The harm caused by overlap therefore is not a gradual drift; it is a phase transition to a completely expelled equilibrium.

\begin{table}[!ht]
\centering
\caption{Mean fraction of prototypes whose center $\hat{x}_j=-b_j/w_j$ lies outside the input range $[0,1]$, by $(N,\text{mask})$. Overlap-active masks reach $1.00$ at every $N\ge 30$.}
\label{tab:expulsion}
\begin{tabular}{lrrrrrrrr}
\toprule
$N$ & \texttt{000} & \texttt{001} & \texttt{010} & \texttt{011} & \texttt{100} & \texttt{101} & \texttt{110} & \texttt{111} \\
\midrule
3 & 0.20 & 0.47 & 0.30 & 0.37 & 0.77 & 0.73 & 0.70 & 0.57 \\
5 & 0.38 & 0.48 & 0.22 & 0.50 & 1.00 & 0.76 & 0.88 & 0.72 \\
10 & 0.45 & 0.50 & 0.41 & 0.51 & 1.00 & 0.90 & 0.96 & 0.86 \\
30 & 0.58 & 0.61 & 0.51 & 0.60 & 1.00 & 1.00 & 1.00 & 0.99 \\
50 & 0.50 & 0.57 & 0.47 & 0.58 & 1.00 & 1.00 & 1.00 & 1.00 \\
100 & 0.52 & 0.60 & 0.52 & 0.60 & 1.00 & 1.00 & 1.00 & 1.00 \\
\bottomrule
\end{tabular}

\end{table}

The expulsion mechanism also explains the specialization-collapse pattern of Figure~\ref{fig:specialization_vs_n} and Figure~\ref{fig:heatmap_specialization}. Once a prototype is expelled, its activation at every training input is small, and it is therefore never the argmin in~\eqref{eq:spec}; only in-hull prototypes can attract a nearest-prototype assignment. With overlap-active masks expelling essentially all prototypes at $N\ge 30$, the inputs effectively share a single residual in-hull neuron and the specialization ratio collapses to $S\approx 1/N$, matching the values $S\le 0.06$ at $N\ge 30$ and $S\approx 0.01$ at $N=100$ that the figures and Table~\ref{tab:specialization_scaling} will report in the next section.

\subsection{Confirmation by a $\tau$-sweep on the Separation-only Mask} 

The mechanism predicts that separation alone, run at much larger $\tau$, should also expel prototypes and degrade reconstruction. Figure~\ref{fig:tau_sweep} and Table~\ref{tab:tau_sweep} report a $\tau$-sweep on mask $\mask{001}$ at $N=30$ over five orders of magnitude. As $\tau$ increases from $10^{-3}$ to $10$, the fraction of prototypes outside $[0,1]$ rises from $0.49$ to $0.79$, the mean signed distance from the input convex hull grows from $0.16$ to $1.36$, and the reconstruction error grows from $0.69$ to $1.78$, thus becoming worse than the typical overlap-active baseline. The paper's chosen value $\tau=10^{-2}$ sits in the well-behaved regime where pairs of prototypes can satisfy the separation constraint without leaving the input range, which is why separation behaves as a mild rather than catastrophic repulsor in our main results.

\begin{figure}[!ht]
\centering
\includegraphics[width=1.0\linewidth]{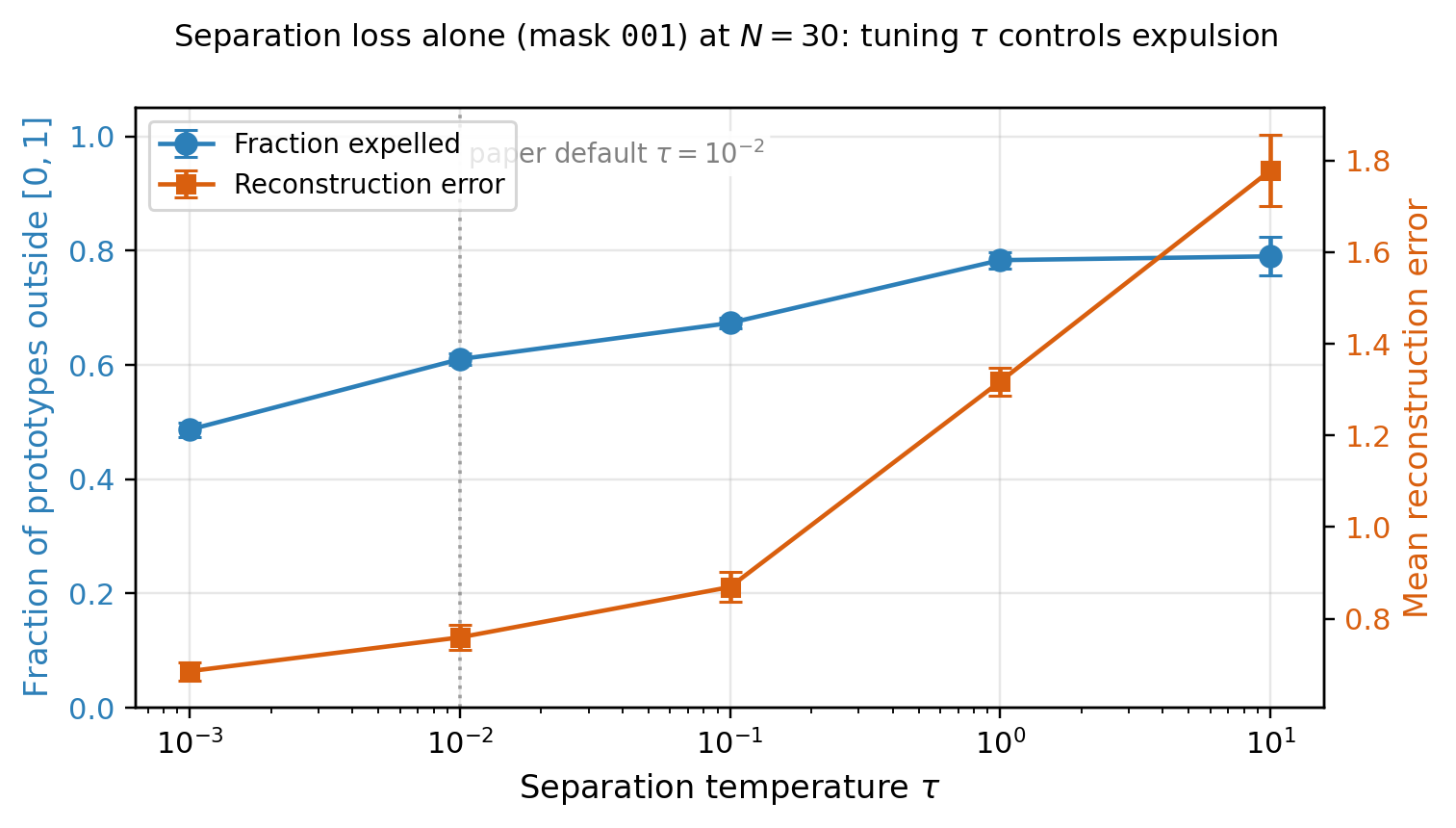}
\caption{$\tau$-sweep on the separation-only mask $\mask{001}$ at $N=30$. As $\tau$ grows, the separation loss admits a stronger expulsion route, the fraction of expelled prototypes rises, and reconstruction error degrades. The paper's default $\tau=10^{-2}$ lies in the well-behaved regime.}
\label{fig:tau_sweep}
\end{figure}

\vspace{-0.3cm}

\begin{table}[!ht]
\centering
\caption{$\tau$-sweep on the separation-only mask $\mask{001}$ at $N=30$. ``frac.\ expelled'' is the average fraction of prototypes outside $[0,1]$; ``dist.\ from hull'' is the average distance from the input convex hull (zero if inside); $E$ is mean reconstruction error; $S$ is the specialization ratio.}
\label{tab:tau_sweep}
\begin{tabular}{rrrrr}
\toprule
$\tau$ & frac.\ expelled & dist.\ from hull & $E$ & $S$ \\
\midrule
0.001 & 0.49 & 0.16 & 0.686 & 0.503 \\
0.01 & 0.61 & 0.31 & 0.760 & 0.403 \\
0.1 & 0.67 & 0.47 & 0.869 & 0.350 \\
1 & 0.78 & 0.92 & 1.318 & 0.230 \\
10 & 0.79 & 1.36 & 1.778 & 0.220 \\
\bottomrule
\end{tabular}

\end{table}

The three structural losses therefore divide along a single axis. Coverage is the only attractor and is the only consistently useful term. Separation and overlap are both repulsors; the difference between them is whether the loss landscape rewards full expulsion (overlap, always) or only at extreme hyperparameter settings (separation, only when $\tau$ is large enough that local separation is insufficient). Adding overlap to any base mask routes that mask to the expelled equilibrium, which is why $\mask{1\!**}$ are systematically worse than their $\mask{0\!**}$ counterparts in Tables~\ref{tab:ablation_core}--\ref{tab:ablation_scaling}.

Specialization behaves differently from reconstruction but supports the same broad conclusion. Figure~\ref{fig:specialization_vs_n}, Figure~\ref{fig:heatmap_specialization}, and Tables~\ref{tab:specialization_core}--\ref{tab:specialization_scaling} show that coverage raises the specialization ratio above the standard baseline from $N=5$ onwards; at $N=3$ the four overlap-free masks all sit close to one another, with the gap well within the run-to-run variance, and the standard baseline happens to attain the column maximum. Separation can increase specialization further, but as already noted, that increase does not translate into the best reconstruction error. The collapse of $S$ toward $1/N$ for overlap-active masks at large $N$ is the same expulsion phenomenon viewed from the specialization side, as discussed in Section~\ref{sec:mechanism}.

\begin{figure}[!ht]
\centering
\includegraphics[width=1.0\linewidth]{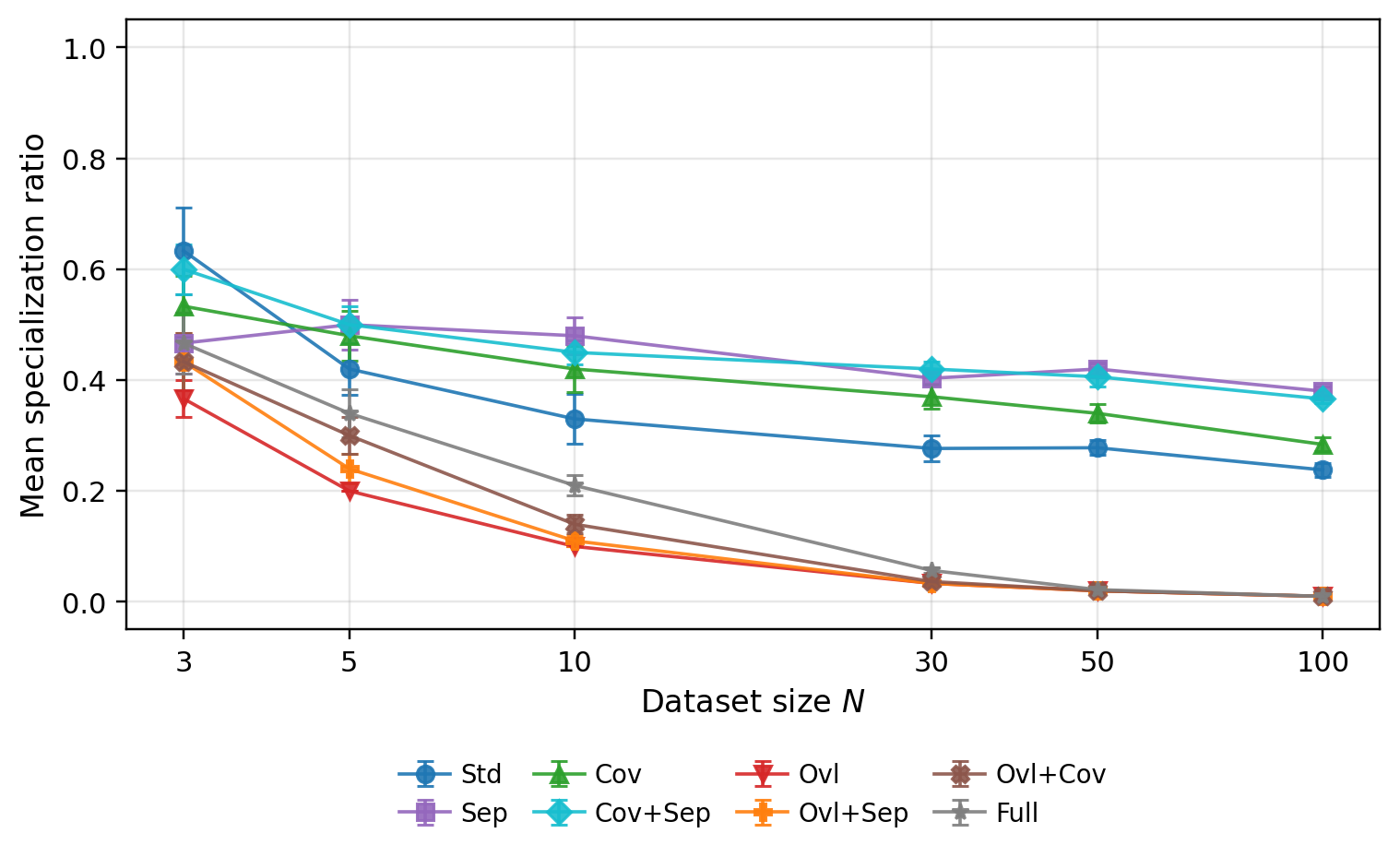}
\caption{Mean specialization ratio versus $N$ for all eight masks; error bars are SEM. Coverage and Cov+Sep raise specialization above the standard baseline at every $N$; overlap-active masks collapse toward zero as $N$ grows.}
\label{fig:specialization_vs_n}
\end{figure}

\begin{figure}[!ht]
\centering
\includegraphics[width=1.0\linewidth]{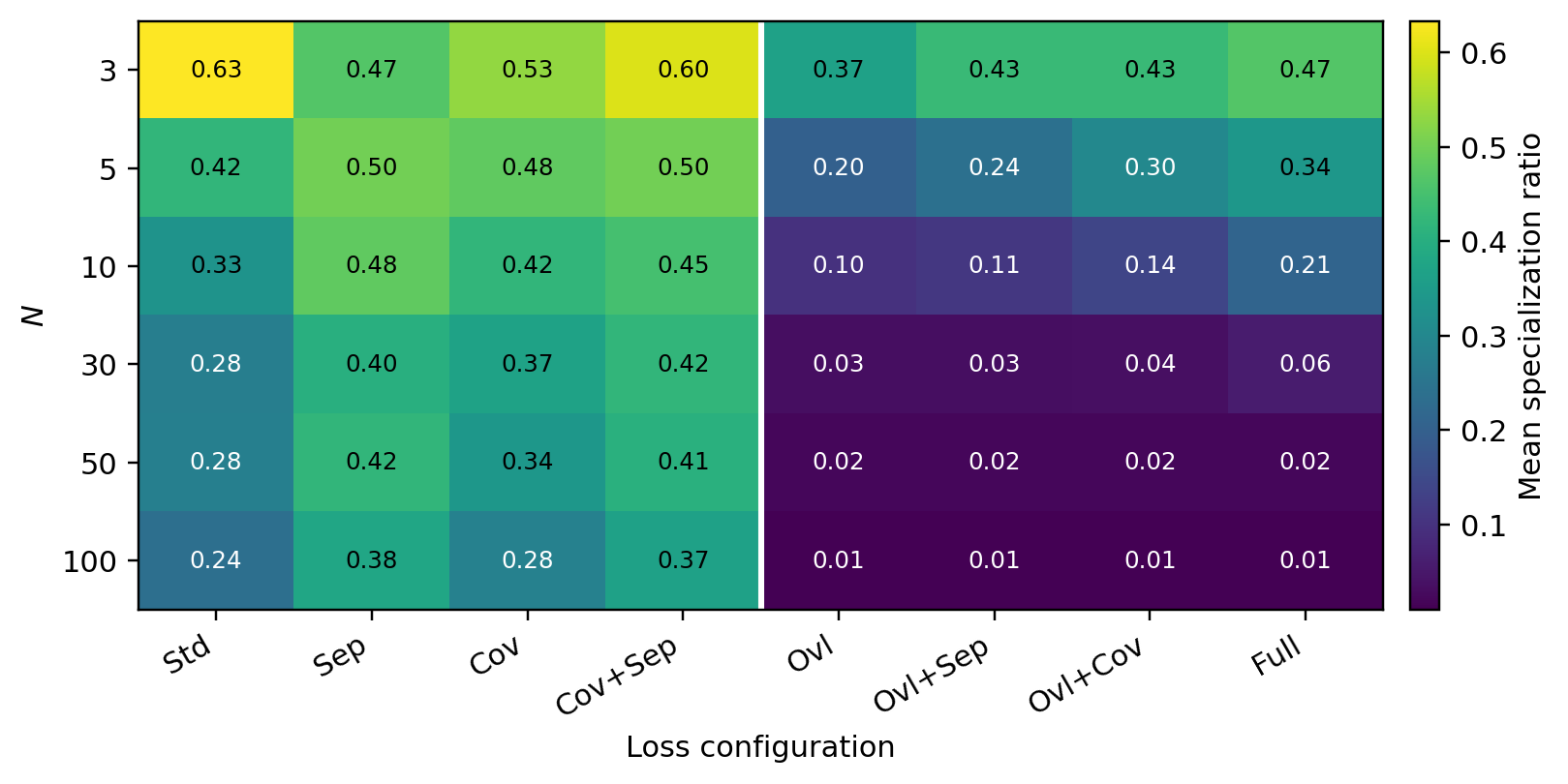}
\caption{Heatmap of mean specialization ratio across masks and dataset sizes (lighter is better). The split between the four overlap-free masks and the four overlap-active masks becomes sharp as $N$ grows.}
\label{fig:heatmap_specialization}
\end{figure}

\begin{table}[!ht]
\centering
\caption{Specialization ratio (mean $\pm$ std) for the standard baseline, coverage-only loss, and full three-loss objective in the core regime. Higher is better. The best value in each column is boldfaced.}
\label{tab:specialization_core}
\begin{tabular}{lrrr}
\toprule
Method & $N=3$ & $N=5$ & $N=10$ \\
\midrule
Std & \textbf{0.633 $\pm$ 0.246} & 0.420 $\pm$ 0.148 & 0.330 $\pm$ 0.142 \\
Cov & 0.533 $\pm$ 0.172 & \textbf{0.480 $\pm$ 0.140} & \textbf{0.420 $\pm$ 0.132} \\
Full & 0.467 $\pm$ 0.172 & 0.340 $\pm$ 0.135 & 0.210 $\pm$ 0.057 \\
\bottomrule
\end{tabular}

\end{table}

\begin{table}[!ht]
\centering
\caption{Specialization ratio (mean $\pm$ std) for the standard baseline, coverage-only loss, and full three-loss objective in the scalability regime. Higher is better. The best value in each column is boldfaced.}
\label{tab:specialization_scaling}
\begin{tabular}{lrrr}
\toprule
Method & $N=30$ & $N=50$ & $N=100$ \\
\midrule
Std & 0.277 $\pm$ 0.074 & 0.278 $\pm$ 0.043 & 0.238 $\pm$ 0.040 \\
Cov & \textbf{0.370 $\pm$ 0.069} & \textbf{0.340 $\pm$ 0.052} & \textbf{0.284 $\pm$ 0.041} \\
Full & 0.057 $\pm$ 0.016 & 0.022 $\pm$ 0.006 & 0.010 $\pm$ 0.000 \\
\bottomrule
\end{tabular}

\end{table}

The pairwise significance analysis sharpens this picture. Tables~\ref{tab:sig_core}--\ref{tab:sig_scaling} report plus--minus matrices for reconstruction error, where ``$+$'' indicates a significant difference after Holm correction and ``$-$'' indicates no significant difference. In the core regime, coverage is significantly better than the worst overlap-driven mask ($\mask{100}$); differences inside the overlap-free family are mostly not significant at $N=3$ and become significant from $N=10$ onwards. In the larger regimes the separation between the overlap-free family ($\mask{000}$, $\mask{001}$, $\mask{010}$, $\mask{011}$) and the overlap-active family ($\mask{100}$, $\mask{101}$, $\mask{110}$, $\mask{111}$) is fully resolved: every cross-family comparison is significant. This supports RQ3: the qualitative ranking is not only preserved at larger $N$, it becomes statistically unambiguous.

\begin{table*}[!h]
\centering
\caption{Pairwise significance matrices for reconstruction error in the core regime. ``$+$'' denotes a significant difference at $\alpha=0.05$ after Holm correction; ``$-$'' denotes no significant difference. Diagonal entries (``$=$'') are self-comparisons.}
\label{tab:sig_core}
\small
\begin{tabular*}{\linewidth}{@{\extracolsep{\fill}}ccc@{}}
$N=3$ & $N=5$ & $N=10$ \\[2pt]

\begin{tabular}{l@{\hspace{2pt}}*{8}{r@{\hspace{2pt}}}}
\toprule
 & \texttt{000} & \texttt{001} & \texttt{010} & \texttt{011} & \texttt{100} & \texttt{101} & \texttt{110} & \texttt{111} \\
\midrule
\texttt{000} & = & - & - & - & + & - & + & - \\
\texttt{001} & - & = & - & - & - & - & - & - \\
\texttt{010} & - & - & = & - & + & - & + & - \\
\texttt{011} & - & - & - & = & + & - & + & - \\
\texttt{100} & + & - & + & + & = & - & - & - \\
\texttt{101} & - & - & - & - & - & = & - & - \\
\texttt{110} & + & - & + & + & - & - & = & - \\
\texttt{111} & - & - & - & - & - & - & - & = \\
\bottomrule
\end{tabular}
 &
\begin{tabular}{l@{\hspace{2pt}}*{8}{r@{\hspace{2pt}}}}
\toprule
 & \texttt{000} & \texttt{001} & \texttt{010} & \texttt{011} & \texttt{100} & \texttt{101} & \texttt{110} & \texttt{111} \\
\midrule
\texttt{000} & = & - & - & - & + & + & + & - \\
\texttt{001} & - & = & - & - & + & + & + & - \\
\texttt{010} & - & - & = & - & + & + & + & + \\
\texttt{011} & - & - & - & = & + & + & + & - \\
\texttt{100} & + & + & + & + & = & - & - & - \\
\texttt{101} & + & + & + & + & - & = & - & - \\
\texttt{110} & + & + & + & + & - & - & = & - \\
\texttt{111} & - & - & + & - & - & - & - & = \\
\bottomrule
\end{tabular}
 &
\begin{tabular}{l@{\hspace{2pt}}*{8}{r@{\hspace{2pt}}}}
\toprule
 & \texttt{000} & \texttt{001} & \texttt{010} & \texttt{011} & \texttt{100} & \texttt{101} & \texttt{110} & \texttt{111} \\
\midrule
\texttt{000} & = & - & - & - & + & + & + & + \\
\texttt{001} & - & = & - & - & + & + & + & - \\
\texttt{010} & - & - & = & - & + & + & + & + \\
\texttt{011} & - & - & - & = & + & + & + & - \\
\texttt{100} & + & + & + & + & = & - & - & + \\
\texttt{101} & + & + & + & + & - & = & - & - \\
\texttt{110} & + & + & + & + & - & - & = & + \\
\texttt{111} & + & - & + & - & + & - & + & = \\
\bottomrule
\end{tabular}

\end{tabular*}
\end{table*}

\begin{table*}[!ht]
\centering
\caption{Pairwise significance matrices for reconstruction error in the scalability regime. Symbols as in Table~\ref{tab:sig_core}.}
\label{tab:sig_scaling}
\small
\begin{tabular*}{\linewidth}{@{\extracolsep{\fill}}ccc@{}}
$N=3$ & $N=5$ & $N=10$ \\[2pt]
\begin{tabular}{l@{\hspace{2pt}}*{8}{r@{\hspace{2pt}}}}
\toprule
 & \texttt{000} & \texttt{001} & \texttt{010} & \texttt{011} & \texttt{100} & \texttt{101} & \texttt{110} & \texttt{111} \\
\midrule
\texttt{000} & = & - & - & - & + & + & + & + \\
\texttt{001} & - & = & - & - & + & + & + & + \\
\texttt{010} & - & - & = & - & + & + & + & + \\
\texttt{011} & - & - & - & = & + & + & + & + \\
\texttt{100} & + & + & + & + & = & - & - & - \\
\texttt{101} & + & + & + & + & - & = & - & - \\
\texttt{110} & + & + & + & + & - & - & = & - \\
\texttt{111} & + & + & + & + & - & - & - & = \\
\bottomrule
\end{tabular}
 &
\begin{tabular}{l@{\hspace{2pt}}*{8}{r@{\hspace{2pt}}}}
\toprule
 & \texttt{000} & \texttt{001} & \texttt{010} & \texttt{011} & \texttt{100} & \texttt{101} & \texttt{110} & \texttt{111} \\
\midrule
\texttt{000} & = & - & - & - & + & + & + & + \\
\texttt{001} & - & = & - & - & + & + & + & + \\
\texttt{010} & - & - & = & - & + & + & + & + \\
\texttt{011} & - & - & - & = & + & + & + & + \\
\texttt{100} & + & + & + & + & = & - & - & - \\
\texttt{101} & + & + & + & + & - & = & - & - \\
\texttt{110} & + & + & + & + & - & - & = & - \\
\texttt{111} & + & + & + & + & - & - & - & = \\
\bottomrule
\end{tabular}
 &
\begin{tabular}{l@{\hspace{2pt}}*{8}{r@{\hspace{2pt}}}}
\toprule
 & \texttt{000} & \texttt{001} & \texttt{010} & \texttt{011} & \texttt{100} & \texttt{101} & \texttt{110} & \texttt{111} \\
\midrule
\texttt{000} & = & - & - & - & + & + & + & + \\
\texttt{001} & - & = & - & - & + & + & + & + \\
\texttt{010} & - & - & = & + & + & + & + & + \\
\texttt{011} & - & - & + & = & + & + & + & + \\
\texttt{100} & + & + & + & + & = & - & - & - \\
\texttt{101} & + & + & + & + & - & = & - & - \\
\texttt{110} & + & + & + & + & - & - & = & - \\
\texttt{111} & + & + & + & + & - & - & - & = \\
\bottomrule
\end{tabular}

\end{tabular*}
\end{table*}

Figure~\ref{fig:representative} complements the aggregate statistics with point-set geometry from a single representative configuration at $N=5$. The three panels show the median-error-supporting run of the standard baseline ($\mask{000}$), the best-performing coverage-only objective ($\mask{010}$), and the full objective ($\mask{111}$) on the same dataset. Originals are marked $o_1,\ldots,o_5$ and reconstructions $r_1,\ldots,r_5$. Coverage places four of the five prototypes inside the input range $[0,1]$, tracking the originals in both $x$ and $y$ within $E=0.292$, and expels only one prototype mildly to $x=1.2$. The standard baseline expels three of five prototypes, with two clustered at the right edge and one well below zero, yielding $E=0.713$. The full objective is dominated by the expulsion mechanism of Section~\ref{sec:mechanism}: only a single prototype remains near the input range, two are expelled to the left, and one is pushed all the way to $x\approx 12$, producing the catastrophic $E=2.923$. The figure therefore shows directly that the ranking of mask families on aggregate error reflects qualitatively different reconstructed geometries, and that overlap-active masks fail not by approximating the wrong $y$ but by relocating their prototypes outside the input region.

\begin{figure}[!ht]
\centering
\includegraphics[width=\linewidth]{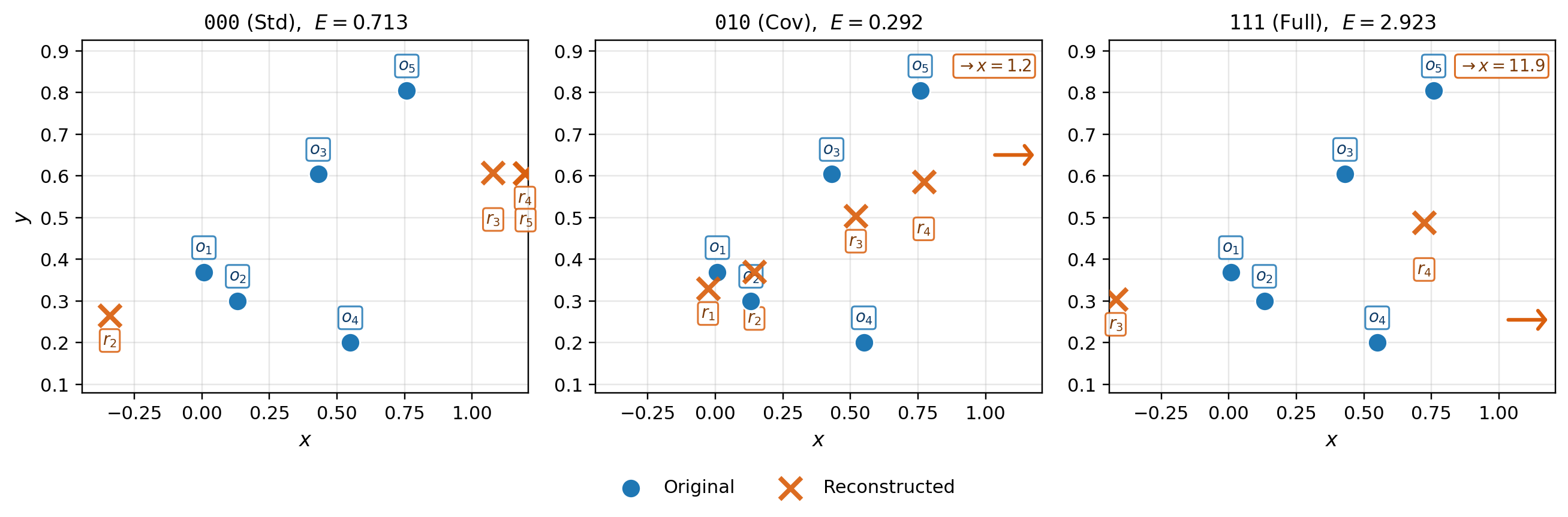}
\caption{Representative reconstructions at $N=5$ for masks $\mask{000}$ (Std), $\mask{010}$ (Cov), and $\mask{111}$ (Full). Originals $o_i$ in blue, reconstructions $r_i$ in orange. Arrows at the right edge mark prototypes expelled outside the plotted range, with their $x$ value. $E$ is the Hungarian-matching reconstruction error.}\label{fig:representative}
\end{figure}

Taken together, the experiments lead to a stable empirical ranking of the three structural terms. Coverage is consistently useful in mean reconstruction error, with effect sizes that are real but modest at the extremes of $N$ and substantial in the mid-range. Separation is mixed: it can increase specialization, but it does not beat coverage-only training in reconstruction accuracy. Overlap is harmful: once activated, it degrades both reconstruction and specialization while leaving the fit essentially intact, and its harm becomes more severe as $N$ grows. The study therefore supports a precise rather than generic notion of prototype-recoverability-aware training: useful biases must preserve enough flexibility for fitting, whereas naive exclusivity constraints overconstrain the hidden layer and misplace its prototypes.

\section{Limitations and Scope}
\label{sec:limitations}

The conclusions above hold under the controlled conditions of Section~\ref{sec:protocol} and should not be extrapolated beyond them. The experiments use one-dimensional synthetic uniform datasets with $N\in\{3,5,10,30,50,100\}$, a Gaussian-activation MLP with width tied to dataset size ($H=N$), a fixed regularisation coefficient $\lambda=10^{-2}$ shared across the three structural terms, $200$ Adam epochs, and $10$ runs per $(N,\text{mask})$ cell. The reconstruction metric is prototype-based and one-dimensional; we make no claim of arbitrary parameter inversion, of resistance against gradient-leakage or membership-inference attacks, or of scalability to high-dimensional inputs and deep architectures. The non-monotonic effect-size profile in Table~\ref{tab:effect_sizes} suggests that conclusions in this family are also conditional on the regularisation coefficient: a single $\lambda$ rules out one direction of mis-scaling but cannot fully separate beneficial structure from coefficient choice. Within these limits the empirical ranking is stable across all tested~$N$.

\section{Reproducibility and Artifact Structure}
\label{sec:repro}

The public release will contain the manuscript source, all figures and tables that appear in the paper, and a Python implementation of the protocol of Section~\ref{sec:protocol}. The implementation generates the synthetic datasets under~\eqref{eq:delta} with the constructive sampler, trains the eight masks under~\eqref{eq:total_loss}, computes the reconstruction error~\eqref{eq:reconstruction_error} by Hungarian assignment, computes the specialization ratio~\eqref{eq:spec}, and regenerates the summary tables, significance matrices, and figures used in Section~\ref{sec:results}. The full grid of $480$ runs takes on the order of one minute on a single CPU core. The master seed is fixed in the run script; per-run seeds are derived deterministically from $(N,\text{dataset id},\text{init id},\text{mask})$ so that any individual configuration can be re-played independently. Final tables and figures are also stored as static artifacts so the manuscript can be rebuilt without re-running the experiment.

\section{Conclusions}
\label{sec:conclusions}

We studied three explicit training biases for hidden-neuron specialization in minimal Gaussian MLPs and evaluated their implications for prototype-based recoverability. The experiments support three conclusions. First, specialization-oriented losses can improve prototype-based reconstruction, but the improvement is not uniform across losses or sizes: coverage helps, separation is mixed, and overlap hurts. Second, the harm caused by overlap penalties is not an optimization failure; it is the result of an expulsion mechanism in which the optimizer satisfies the loss by pushing prototype centers outside the convex hull of the training inputs. Coverage is the only attractor among the three structural losses we study, and is also the only one that consistently helps; separation and overlap are both repulsors and admit the same expulsive equilibrium, with overlap reaching it at the paper's nominal hyperparameter choice while separation reaches it only at much larger $\tau$. Third, the full combined objective does not outperform the best partial objective, and its degradation becomes clearer rather than weaker as $N$ grows, because adding overlap to any base mask is enough to route that mask to the expelled regime.

These findings matter beyond the specific synthetic setting. They show that prototype-recoverability-aware training is not simply a matter of adding more structural terms to the objective. The geometry induced by the losses must remain compatible with the optimization problem, and in particular a repulsive loss without a compensating attractor will admit a degenerate minimum at infinity. The phenomenon is closely related to the attraction--repulsion gradient imbalance that Xie et al.\ identify in long-tailed classification~\cite{xie2023neural}, and the connection suggests that the design principle ``every repulsor needs a compatible attractor'' carries across settings as different as classification with class means and regression with prototypes. The same lens may inform design-time certification: structural training biases that shape an auditable latent geometry could complement post hoc verification of properties such as general ethical compliance, by making relevant features of the trained model evident from its weights. Future work on prototype-recoverability-aware training at larger scales and in richer architectures, including iterative refinement procedures that operate on the prototype set produced by this baseline and meta-optimization of the per-term coefficients $\lambda_o,\lambda_c,\lambda_s$ against the specialization ratio $S$, can use this principle as a starting point.


\section*{Appendix A. Proofs}

\setcounter{theorem}{0}

\begin{theorem}[Coverage cannot reward expulsion]
\label{thm:coverage_hull}
Let $D=\{x_i\}_{i=1}^{N}\subset[0,1]$ be the training inputs and
let $\hat{x}_j=-b_j/w_j$ be the $j$-th prototype.
If $\hat{x}_j\notin[0,1]$ for some $j$, then
moving $\hat{x}_j$ further away from $[0,1]$
does not decrease $L_{\mathrm{coverage}}$.
\end{theorem}

\begin{proof}
Write $L_{\mathrm{coverage}}=\sum_{i=1}^{N}\min_{k}(x_i-\hat{x}_k)^2$.
Fix $j$ and suppose $\hat{x}_j\notin[0,1]$; without loss of generality
assume $\hat{x}_j>1$ (the case $\hat{x}_j<0$ is symmetric).
For every $x_i\in[0,1]$ we have $|x_i-\hat{x}_j|\ge \hat{x}_j-1>0$.
Let $\hat{x}_j'=\hat{x}_j+\varepsilon$ with $\varepsilon>0$, while all other prototypes are held fixed.
For any training point $x_i$,
\[
  |x_i-\hat{x}_j'|=|x_i-\hat{x}_j|+\varepsilon\cdot\mathrm{sgn}(\hat{x}_j'-x_i)
  \;=\;|x_i-\hat{x}_j|+\varepsilon\;\ge\;|x_i-\hat{x}_j|,
\]
since $x_i<1<\hat{x}_j<\hat{x}_j'$ implies $\hat{x}_j'-x_i>0$.
Hence the term $\min_k(x_i-\hat{x}_k)^2$ can only stay equal or
increase when $\hat{x}_j$ is replaced by $\hat{x}_j'$, for every $i$.
Summing over $i$ gives $L_{\mathrm{coverage}}(\hat{x}_j')\ge L_{\mathrm{coverage}}(\hat{x}_j)$.
\end{proof}

\begin{theorem}[Overlap loss vanishes with separation]
\label{thm:overlap_vanishes}
Let $\Delta=\min_{j\neq k}|\hat{x}_j-\hat{x}_k|$ be the minimum
pairwise distance between prototypes, and assume $|w_j|\ge w_{\min}>0$
for all $j$.  Then
\[
  L_{\mathrm{overlap}}\;\le\;
  \binom{N}{2}N\,e^{-w_{\min}^2\Delta^2/2}\;\xrightarrow{\Delta\to\infty}\;0.
\]
\end{theorem}

\begin{proof}
Using $\sigma(w_jx_i+b_j)=\exp\!\bigl(-w_j^2(x_i-\hat x_j)^2\bigr)$,
each term in~\eqref{eq:lov} factors as
\[
\sigma(w_jx_i+b_j)\,\sigma(w_kx_i+b_k)
=\exp\!\Bigl(-w_j^2(x_i-\hat x_j)^2-w_k^2(x_i-\hat x_k)^2\Bigr).
\]
Since $|w_j|,|w_k|\ge w_{\min}$,
\[
\sigma(w_jx_i+b_j)\,\sigma(w_kx_i+b_k)
\;\le\;
\exp\!\Bigl(-w_{\min}^2\bigl[(x_i-\hat x_j)^2+(x_i-\hat x_k)^2\bigr]\Bigr).
\]
The algebraic identity
\[
(x_i-\hat x_j)^2+(x_i-\hat x_k)^2
=2\Bigl(x_i-\tfrac{\hat x_j+\hat x_k}{2}\Bigr)^{\!2}
+\tfrac{1}{2}(\hat x_j-\hat x_k)^2
\;\ge\;\tfrac{1}{2}(\hat x_j-\hat x_k)^2
\;\ge\;\tfrac{\Delta^2}{2}
\]
then yields
\[
\sigma(w_jx_i+b_j)\,\sigma(w_kx_i+b_k)
\;\le\;\exp\!\bigl(-w_{\min}^2\Delta^2/2\bigr).
\]
Summing over the $\binom{N}{2}$ pairs and the $N$ inputs gives
$L_{\mathrm{overlap}}\le\binom{N}{2}N\exp(-w_{\min}^2\Delta^2/2)\to 0$
as $\Delta\to\infty$.
\end{proof}

\begin{theorem}[Separation loss vanishes with separation]
\label{thm:separation_vanishes}
Let $\Delta=\min_{j\neq k}|\hat{x}_j-\hat{x}_k|$ and let $\tau>0$.
Then
\[
  L_{\mathrm{separation}}\;\le\;
  \binom{N}{2}\exp\!\Bigl(-\tfrac{\Delta^2}{\tau}\Bigr)
  \;\xrightarrow{\Delta\to\infty}\;0.
\]
\end{theorem}

\begin{proof}
Each term in~\eqref{eq:lsep} satisfies
\[
\exp\!\left(-\frac{(\hat x_j-\hat x_k)^2}{\tau}\right)
\;\le\;
\exp\!\left(-\frac{\Delta^2}{\tau}\right),
\]
because $|\hat x_j-\hat x_k|\ge\Delta$ for all $j\neq k$.
Summing over the $\binom{N}{2}$ pairs and taking $\Delta\to\infty$
completes the proof.
\end{proof}

\section*{Declaration of generative AI and AI-assisted technologies }

During the preparation of this work the authors used Claude Opus 4.7 in order to summarise results in graphics and cleaning the English expressions in the paper. After using this tool, the author reviewed and edited the whole content as needed and take full responsibility for the content of the published article.

\bibliographystyle{elsarticle-num}
\bibliography{references}

\end{document}